\documentclass{article}
\usepackage{fullpage}
\usepackage{amsfonts,amssymb,amsmath}
\usepackage{bbm}

\usepackage{xcolor}

\usepackage{graphicx}

\usepackage{subcaption} 

\usepackage{algorithm}
\usepackage{algorithmic}

\usepackage{amsthm}
\usepackage{thmtools}
\usepackage{thm-restate}

\newtheorem{theorem}{Theorem}
\newtheorem*{theorem*}{Theorem}

\newtheorem{lemma}[theorem]{Lemma}
\newtheorem*{lemma*}{Lemma}

\newtheorem*{definition*}{Definition}

\newcommand{\E}{\mathbb{E}}

\newcommand{\R}{\mathbb{R}}
\newcommand{\X}{\mathcal{X}}

\newcommand{\HH}{\mathcal{H}}
\newcommand{\D}{\mathcal{D}}

\def\pr{{\rm Pr}}

\newcommand{\ind}{ \mathbbm{1}}

\usepackage{float}

\title{Diameter-Based Active Learning}
\author{Christopher Tosh \and Sanjoy Dasgupta}

\author{Christopher Tosh \\
 University of California, San Diego \\
{\tt{ctosh@cs.ucsd.edu}}
\and
Sanjoy Dasgupta \\
 University of California, San Diego \\
{\tt{dasgupta@cs.ucsd.edu}}
}

\begin{document}

\maketitle

\begin{abstract} 
To date, the tightest upper and lower-bounds for the active learning of general concept classes have been in terms of a parameter of the learning problem called the \emph{splitting index}. We provide, for the first time, an efficient algorithm that is able to realize this upper bound, and we empirically demonstrate its good performance.
\end{abstract} 

\section{Introduction}
\label{section: Introduction}

In many situations where a classifier is to be learned, it is easy to collect unlabeled data but costly to obtain labels. This has motivated the {\it pool-based active learning} model, in which a learner has access to a collection of unlabeled data points and is allowed to ask for individual labels in an adaptive manner. The hope is that choosing these queries intelligently will rapidly yield a low-error classifier, much more quickly than with random querying. A central focus of active learning is developing efficient querying strategies and understanding their label complexity.

Over the past decade or two, there has been substantial progress in developing such rigorously-justified active learning schemes for general concept classes. For the most part, these schemes can be described as {\it mellow}: rather than focusing upon maximally informative points, they query any point whose label cannot reasonably be inferred from the information received so far. It is of interest to develop more aggressive strategies with better label complexity.

An exception to this general trend is the aggressive strategy of \cite{dasgupta-splitting-index}, whose label complexity is known to be optimal in its dependence on a key parameter called the {\it splitting index}. However, this strategy has been primarily of theoretical interest because it is difficult to implement algorithmically. In this paper, we introduce a variant of the methodology that yields efficient algorithms. We show that it admits roughly the same label complexity bounds as well as having promising experimental performance.

As with the original splitting index result, we operate in the {\it realizable} setting, where data can be perfectly classified by some function $h^*$ in the hypothesis class $\HH$. At any given time during the active learning process, the remaining candidates---that is, the elements of $\HH$ consistent with the data so far---are called the {\it version space}. The goal of aggressive active learners is typically to pick queries that are likely to shrink this version space rapidly. But what is the right notion of size? Dasgupta~\cite{dasgupta-splitting-index} pointed out that the {\it diameter} of the version space is what matters, where the distance between two classifiers is taken to be the fraction of points on which they make different predictions. Unfortunately, the diameter is a difficult measure to work with because it cannot, in general, be decreased at a steady rate. Thus the earlier work used a procedure that has quantifiable label complexity but is not conducive to implementation.

We take a fresh perspective on this earlier result. We start by suggesting an alternative, but closely related, notion of the size of a version space: the {\it average} pairwise distance between hypotheses in the version space, with respect to some underlying probability distribution $\pi$ on $\HH$. This distribution $\pi$ can be arbitrary---that is, there is no requirement that the target $h^*$ is chosen from it---but should be chosen so that it is easy to sample from. When $\HH$ consists of linear separators, for instance, a good choice would be a log-concave density, such as a Gaussian. 

At any given time, the next query $x$ is chosen roughly as follows:
\begin{itemize}
\item Sample a collection of classifiers $h_1, h_2, \ldots, h_m$ from $\pi$ restricted to the current version space $V$. 
\item Compute the distances between them; this can be done using just the unlabeled points.
\item Any candidate query $x$ partitions the classifiers $\{h_i\}$ into two groups: those that assign it a $+$ label (call these $V_x^+$) and those that assign it a $-$ label (call these $V_x^-$). Estimate the average-diameter after labeling $x$ by the sum of the distances between classifiers $h_i$ within $V_x^+$, or those within $V_x^-$, whichever is larger.
\item Out of the pool of unlabeled data, pick the $x$ for which this diameter-estimate is smallest.
\end{itemize}
This is repeated until the version space has small enough average diameter that a random sample from it is very likely to have error less than a user-specified threshold $\epsilon$. We show how all these steps can be achieved efficiently, as long as there is a sampler for $\pi$.

Dasgupta~\cite{dasgupta-splitting-index} pointed out that the label complexity of active learning depends on the underlying distribution, the amount of unlabeled data (since more data means greater potential for highly-informative points), and also the target classifier $h^*$. That paper identifies a parameter called the {\it splitting index} $\rho$ that captures the relevant geometry, and gives upper bounds on label complexity that are proportional to $1/\rho$, as well as showing that this dependence is inevitable. For our modified notion of diameter, a different {\it averaged} splitting index is needed. However, we show that it can be bounded by the original splitting index, with an extra multiplicative factor of $\log (1/\epsilon)$; thus all previously-obtained label complexity results translate immediately for our new algorithm.

\section{Related work}

The theory of active learning has developed along several fronts.

One of these is {\it nonparametric} active learning, where the learner starts with a pool of unlabeled points, adaptively queries a few of them, and then fills in the remaining labels. The goal is to do this with as few errors as possible. (In particular, the learner does not return a classifier from some predefined parametrized class.) One scheme begins by building a neighborhood graph on the unlabeled data, and propagating queried labels along the edges of this graph \cite{zhu-gl-semi-supervised-graph, cesa-bianchi-gv-active-graph, dasarathy-nz-active-graph}. Another starts with a hierarchical clustering of the data and moves down the tree, sampling at random until it finds clusters that are relatively pure in their labels \cite{dasgupta-h-active-hierarchical}. The label complexity of such methods have typically be given in terms of smoothness properties of the underlying data distribution \cite{castro-n-active-minimax, kpotufe-ub-active-hierarchical}. 

Another line of work has focused on active learning of linear separators, by querying points close to the current guess at the decision boundary~\cite{balcan-bz-active-margin, dasgupta-km-active-perceptron, balcan-l-active-linear-separators}. Such algorithms are close in spirit to those used in practice, but their analysis to date has required fairly strong assumptions to the effect that the underlying distribution on the unlabeled points is logconcave. Interestingly, regret guarantees for online algorithms of this sort can be shown under far weaker conditions \cite{cesa-bianchi-gz-active-online-linear}.

The third category of results, to which the present paper belongs, considers active learning strategies for general concept classes $\HH$. Some of these schemes~\cite{cohn1994improving,dasgupta-mh-agnostic-active, beygelzimer-dl-iwal, balcan-bl-active-agnostic, zhang-c-active-agnostic} are fairly mellow in the sense described earlier, using generalization bounds to gauge which labels can be inferred from those obtained so far. The label complexity of these methods can be bounded in terms of a quantity known as the disagreement coefficient~\cite{hanneke-A-2-analysis}. In the realizable case, the canonical such algorithm is that of \cite{cohn1994improving}, henceforth referred to as CAL. Other methods use a prior distribution $\pi$ over the hypothesis class, sometimes assuming that the target classifier is a random draw from this prior. These methods typically aim to shrink the mass of the version space under $\pi$, either greedily and explicitly~\cite{dasgupta-greedy-active, guillory-b-active-average-case-costs, golovin-kr-bayesian-active} or implicitly~\cite{freund-sst-qbc-analysis}. Perhaps the most widely-used of these methods is the latter, query-by-committee, henceforth QBC. As mentioned earlier, shrinking $\pi$-mass is not an optimal strategy if low misclassification error is the ultimate goal. In particular, what matters is not the prior mass of the remaining version space, but rather how {\it different} these candidate classifiers are from each other. This motivates using the diameter of the version space as a yardstick, which was first proposed in \cite{dasgupta-splitting-index} and is taken up again here.

\section{Preliminaries}
\label{section: Preliminaries}
Consider a binary hypothesis class $\HH$, a data space $\X$, and a distribution $\D$ over $\X$. For mathematical convenience, we will restrict ourselves to finite hypothesis classes. (We can do this without loss of generality when $\HH$ has finite VC dimension, since we only use the predictions of hypotheses on a pool of unlabeled points; however, we do not spell out the details of this reduction here.) The \emph{hypothesis distance} induced by $\D$ over $\HH$ is the pseudometric 
\[ d(h, h') \ := \ \pr_{x \sim \D}(h(x) \neq h'(x)). \] 
Given a point $x \in \X$ and a subset $V \subset \HH$, denote
\[ V_x^+ \ = \ \{ h \in V \, : \, h(x) = 1 \} \]
and $V_x^- = V \setminus V_x^+$. Given a sequence of data points $x_1, \ldots, x_n$ and a target hypothesis $h^*$, the induced \emph{version space} is the set of hypotheses that are consistent with the target hypotheses on the sequence, i.e.
\[ \{ h \in \HH \, : \, h(x_i) = h^*(x_i) \text{ for all } i=1,\ldots, n \}. \]

\subsection{Diameter and the Splitting Index}
The \emph{diameter} of a set of hypotheses $V \subset \HH$ is the maximal distance between any two hypotheses in $V$, i.e. 
\[ \text{diam}(V) := \max_{h, h' \in V} d(h, h'). \]
Without any prior information, any hypothesis in the version space could be the target. Thus the worst case error of any hypothesis in the version space is the diameter of the version space. The splitting index roughly characterizes the number of queries required for an active learning algorithm to reduce the diameter of the version space below $\epsilon$.

While reducing the diameter of a version space $V \subset \HH$, we will sometimes identify pairs of hypotheses $h,h' \in V$ that are far apart and therefore need to be separated. We will refer to $\{h,h'\}$ as an {\it edge}. Given a set of edges $E = \{ \{h_1, h'_1\}, \ldots , \{h_n, h'_n \} \} \subset {\HH \choose 2 }$, we say a data point $x$ $\rho$-splits $E$ if querying $x$ separates at least a $\rho$ fraction of the pairs, that is, if
\[ \max \left\{ \left| E_x^+|,|E_x^- \right|  \right\} \ \leq \ (1-\rho)|E| \]
where $E_x^+ = E \cap {\HH_x^+ \choose 2}$ and similarly for $E_x^-$. 
When attempting to get accuracy $\epsilon > 0$, we need to only eliminate edge of length greater than $\epsilon$. Define
\[ E_\epsilon \ = \ \{ \{ h, h' \} \in E \, : \, d(h, h') > \epsilon \} . \]
The \emph{splitting index} of a set $V \subset \HH$ is a tuple $(\rho, \epsilon, \tau)$ such that for all finite edge-sets $E \subset {V \choose 2}$,
\[ \pr_{x \sim \D}(x \, \rho\text{-splits } E_\epsilon) \ \geq \  \tau. \]
The following theorem, due to Dasgupta \cite{dasgupta-splitting-index}, bounds the sample complexity of active learning in terms of the splitting index. The $\tilde{O}$ notation hides polylogarithmic factors in $d$, $\rho$, $\tau$, $\log 1/\epsilon$, and the failure probability $\delta$.
\begin{theorem}[Dasgupta 2005]
Suppose $\HH$ is a hypothesis class with splitting index $(\rho, \epsilon, \tau)$. Then to learn a hypothesis with error $\epsilon$,
\begin{enumerate}
	\item[(a)] any active learning algorithm with $\leq 1/\tau$ unlabeled samples must request at least $1/\rho$ labels, and 
	\item[(b)] if $\HH$ has VC-dimension $d$, there is an active learning algorithm that draws $\tilde{O}(d/(\rho \tau) \log^2 (1/\epsilon))$ unlabeled data points and requests $\tilde{O}((d/\rho) \log^2 (1/\epsilon))$ labels.
\end{enumerate}
\end{theorem}

Unfortunately, the only known algorithm satisfying (b) above is intractable for all but the simplest hypothesis classes: it constructs an $\epsilon$-covering of the hypothesis space and queries points which whittle away at the diameter of this covering. To overcome this intractability, we consider a slightly more benign setting in which we have a samplable prior distribution $\pi$ over our hypothesis space $\HH$.

\subsection{An Average Notion of Diameter}
With a prior distribution, it makes sense to shift away from the worst-case to the average-case. We define the \emph{average diameter} of a subset $V \subset \HH$ as the expected distance between two hypotheses in $V$ randomly drawn from $\pi$, i.e. 
\[ \Phi(V) := \E_{h, h' \sim \pi|_V}[d(h, h')] \]
where $\pi|_V$ is the conditional distribution induced by restricting $\pi$ to $V$, that is, $\pi|_V(h) = \pi(h)/\pi(V)$ for $h \in V$. 

Intuitively, a version space with very small average diameter ought to put high weight on hypotheses that are close to the true hypothesis. Indeed, given a version space $V$ with $h^* \in V$, the following lemma shows that if $\Phi(V)$ is small enough, then a low error hypothesis can be found by two popular heuristics: random sampling and MAP estimation. 
\begin{restatable}{lemma}{SmallAverageDiameterImpliesGoodEstimateLemma}
\label{lemma: small average diameter implies good estimate lemma}
Suppose $V \subset \HH$ contains $h^*$. Pick $\epsilon > 0$.
\begin{itemize}
\item[(a)] (Random sampling) If $\Phi(V) \leq \epsilon \, \pi|_V(h^*)$ then $\E_{h \sim \pi|_V} [d(h^*,h)] \leq \epsilon$.
\item[(b)] (MAP estimation) Write $p_{map} = \max_{h \in V} \pi|_V(h)$. Pick $0 < \alpha < p_{map}$. If 
$$\Phi(V) \ \leq \ 2 \epsilon  \left( \min \{ \pi|_V(h^*), p_{map}-\alpha \} \right)^2,$$
 then $d(h^*,h) \leq \epsilon$ for any $h$ with $\pi|_V(h) \geq p_{map} - \alpha$.
\end{itemize}
\end{restatable}
\begin{proof}
Part (a) follows from
\begin{align*}
\Phi(V) \ = \ \E_{h,h' \sim \pi|_V}[d(h,h')] \ \geq \ \pi|_V(h^*) \E_{h \sim \pi|_V}[d(h^*,h)] .
\end{align*}
For (b), take $\delta =  \min(\pi|_V(h^*), p_{map} - \alpha)$ and define $V_{\pi, \delta} = \{ h \in V \, : \, \pi|_V(h) \geq \delta  \}$. Note that $V_{\pi,\delta}$ contains $h^*$ as well as any $h \in V$ with $\pi|_V(h) \geq p_{map} - \alpha$.

We claim $\text{diam}(V_{\pi, \delta})$ is at most $\epsilon$. Suppose not. Then there exist $h_1, h_2 \in V_{\pi, \delta}$ satisfying $d(h_1, h_2) > \epsilon$, implying
\begin{align*}
\Phi(V) &= \E_{h,h' \sim \pi|_V}[d(h,h')] \\
&\geq 2 \cdot \pi|_V(h_1) \cdot \pi|_V(h_2) \cdot d(h_1, h_2) \ > \ 2 \delta^2 \epsilon.
\end{align*} 
But this contradicts our assumption on $\Phi(V)$. Since both $h, h^* \in V_{\pi, \delta}$, we have (b).
\end{proof}

\subsection{An Average Notion of Splitting}
We now turn to defining an average notion of splitting. A data point $x$ \emph{$\rho$-average splits} $V$ if
\begin{align*}
\max \left\{ \frac{\pi(V_x^+)^2}{\pi(V)^2} \Phi(V_x^+),  \frac{\pi(V_x^-)^2}{\pi(V)^2} \Phi(V_x^-) \right\} \ \leq \ (1-\rho) \Phi(V).  
\end{align*}
And we say a set $S \subset \HH$ has \emph{average splitting index} $(\rho, \epsilon, \tau)$ if for any subset $V \subset S$ such that $\Phi(V) > \epsilon$, 
\[ \pr_{x \sim \D}\left( x \, \rho\text{-average splits } V \right) \ \geq \ \tau. \]
Intuitively, average splitting refers to the ability to significantly decrease the potential function 
$$\pi(V)^2 \Phi(V) \ = \ \E_{h,h' \sim \pi}[\ind(h,h' \in V) \, d(h,h')]$$
with a single query.

While this potential function may seem strange at first glance, it is closely related to the original splitting index. The following lemma, whose proof is deferred to Section~\ref{section: Proof}, shows the splitting index bounds the average splitting index for any hypothesis class.

\begin{restatable}{lemma}{SplittingImpliesAverageSplittingLemma}
\label{lemma: splitting index implies average splitting index lemma}
Let $\pi$ be a probability measure over a hypothesis class $\HH$. If $\HH$ has splitting index $(\rho, \epsilon, \tau)$, then it has average splitting index $(\frac{\rho}{4 \lceil \log(1/\epsilon) \rceil}, 2\epsilon, \tau)$.
\end{restatable}

Dasgupta~\cite{dasgupta-splitting-index} derived the splitting indices for several hypothesis classes, including intervals and homogeneous linear separators. Lemma~\ref{lemma: splitting index implies average splitting index lemma} implies average splitting indices within a $\log(1/\epsilon)$ factor in these settings.

Moreover, given access to samples from $\pi|_V$, we can easily estimate the quantities appearing in the definition of average splitting. For an edge sequence $E = (\{h_1, h'_1\}, \ldots , \{h_n, h'_n \})$, define 
\[ \psi(E) := \sum_{i=1}^n d(h_i, h'_i). \]
When $h_i, h'_i$ are i.i.d. draws from $\pi|_V$ for all $i=1, \ldots, n$, which we denote $E \sim (\pi|_V)^{2 \times n}$, the random variables $\psi(E)$, $\psi(E_x^-)$, and $\psi(E_x^+)$ are unbiased estimators of the quantities appearing in the definition of average splitting.  
\begin{lemma}
\label{lemma: expectation lemma}
Given $E \sim (\pi|_V)^{2 \times n}$, we have
\begin{itemize}
	\item $\E\left[\frac{1}{n}\psi(E) \right] = \Phi(V)$ and
	\item $ \E\left[ \frac{1}{n}\psi(E_x^+) \right] = \frac{\pi(V_x^+)^2}{\pi(V)^2} \Phi(V_x^+)$ for any $x \in \X$. Similarly for $E_x^-$ and $V_x^-$.
\end{itemize}
\end{lemma}
\begin{proof}
From definitions and linearity of expectations, it is easy to observe $\E[\psi(E)] = n \, \Phi(V)$. By the independence of $h_i, h'_i$, we additionally have
\begin{align*}
\E \left[ \frac{1}{n} \psi(E_x^+) \right] &= \frac{1}{n} \E \left[ \sum_{\{ h_i, h_i' \} \in E_x^+} d(h_i, h'_i) \right] \\ \displaybreak[3]
&= \frac{1}{n} \E \left[ \sum_{\{ h_i, h_i' \} \in E} \ind[h_i \in V_x^+] \, \ind[h'_i \in V_x^+] \, d(h_i, h'_i) \right] \\ \displaybreak[3]
&= \frac{1}{n} \sum_{\{ h_i, h_i' \} \in E} \left(  \frac{\pi(V_x^+)}{\pi(V)} \right)^2 \E\left[ d(h_i, h'_i) \, | \, h_i, h'_i \in V_x^+\right] \\ \displaybreak[3]
&= \left(  \frac{\pi(V_x^+)}{\pi(V)} \right)^2 \Phi(V_x^+). \qedhere
\end{align*}
\end{proof}
\paragraph{Remark:} It is tempting to define average splitting in terms of the average diameter as \[ \max \{ \Phi(V_x^+), \Phi(V_x^-) \} \ \leq \ (1- \rho) \Phi(V). \] However, this definition does not satisfy a nice relationship with the splitting index. Indeed, there exist hypothesis classes $V$ for which there are many points which $1/4$-split $E$ for any $E \subset {V \choose 2}$ but for which every $x \in \X$ satisfies 
\[ \max \{ \Phi(V_x^+), \Phi(V_x^-) \} \ \approx \ \Phi(V).\] 
This observation is formally proven in the appendix.

\section{An Average Splitting Index Algorithm}
\label{section: Average Splitting Algorithm}
Suppose we are given a version space $V$ with average splitting index $(\rho, \epsilon, \tau)$. If we draw $\tilde{O}(1/\tau)$ points from the data distribution then, with high probability, one of these will $\rho$-average split $V$. Querying that point will result in a version space $V'$ with significantly smaller potential {$\pi(V')^2 \Phi(V')$}.

If we knew the value $\rho$ a priori, then Lemma~\ref{lemma: expectation lemma} combined with standard concentration bounds \cite{hoeffding-inequality, angluin-v-multiplicative-chernoff} would give us a relatively straightforward procedure to find a good query point: 
\begin{enumerate}
	\item Draw $E' \sim (\pi|_V)^{2\times M}$ and compute the empirical estimate $ \widehat{\Phi}(V) = \frac{1}{M} \psi(E')$.
	\item Draw $E \sim (\pi|_V)^{2\times N}$ for $N$ depending on $\rho$ and $\widehat{\Phi}$.
    \item For suitable $M$ and $N$, it will be the case that with high probability, for some $x$, 
\[ \frac{1}{N} \max \left\{ \psi(E_x^+), \psi(E_x^-)  \right\} \ \approx \ (1-\rho)\widehat{\Phi}. \]
Querying that point will decrease the potential.  
\end{enumerate}

However, we typically would not know the average splitting index ahead of time. Moreover, it is possible that the average splitting index may change from one version space to the next. In the next section, we describe a query selection procedure that {adapts} to the splittability of the current version space.

\subsection{Finding a Good Query Point}
Algorithm~\ref{algorithm: query selection procedure}, which we term {\sc select}, is our query selection procedure. It takes as input a sequence of data points $x_1, \ldots, x_m$, at least one of which $\rho$-average splits the current version space, and with high probability finds a data point that $\rho/8$-average splits the version space.

{\sc select} proceeds by positing an optimistic estimate of $\rho$, which we denote $\widehat{\rho}_t$, and successively halving it until we are confident that we have found a point that $\widehat{\rho}_t$-average splits the version space. In order for this algorithm to succeed, we need to choose $n_t$ and $m_t$ such that with high probability (1) $\widehat{\Phi}_t$ is an accurate estimate of $\Phi(V)$ and (2) our halting condition will be true if $\widehat{\rho}_t$ is within a constant factor of $\rho$ and false otherwise. The following lemma, whose proof is in the appendix, provides such choices for $n_t$ and $m_t$.

\begin{restatable}{lemma}{SelectionProcedureLemma}
\label{lemma: selection procedure lemma}
Let $\rho, \epsilon, \delta_0 >0$ be given. Suppose that version space $V$ satisfies $\Phi(V) > \epsilon$. In {\sc select}, fix a round $t$ and data point $x \in \X$ that {\it exactly} $\rho$-average splits $V$ (that is, $\max \{ \pi|_V(V_x^+)^2 \Phi(V_x^+), \ \pi|_V(V_x^-)^2 \Phi(V_x^-) \} = (1-\rho)\Phi(V)$). If $m_t \geq \frac{48}{\widehat{\rho}_t^2 \epsilon} \log \frac{4}{\delta_0}$ and $n_t \geq \max \left\{ \frac{32}{\widehat{\rho}_t^2 \widehat{\Phi}_t}, \frac{40}{\widehat{\Phi}_t^2}  \right\}\log \frac{4}{\delta_0}$ then with probability $1-\delta_0$, 
\begin{itemize}
	\item[(a)] $\widehat{\Phi}_t \geq (1-\widehat{\rho}_t/4)\Phi(V)$;
	\item[(b)] if $\rho \leq \widehat{\rho}_t/2$, then $\frac{1}{n_t}\max \left\{ \psi(E_x^+), \psi(E_x^-)\right\} > (1- \widehat{\rho}_t) \widehat{\Phi}_t$; and
	\item[(c)] if $\rho \geq 2\widehat{\rho}_t$, then $\frac{1}{n_t}\max \left\{ \psi(E_x^+), \psi(E_x^-)\right\} \leq (1- \widehat{\rho}_t) \widehat{\Phi}_t .$
\end{itemize}
\end{restatable}

Given the above lemma, we can establish a bound on the number of rounds and the total number of hypotheses \textsc{select} needs to find a data point that $\rho/8$-average splits the version space.

\begin{theorem}
\label{theorem: selection procedure theorem}
Suppose that {\sc select} is called with a version space $V$ with $\Phi(V) \geq \epsilon$ and a collection of points $x_1, \ldots, x_m$ such that at least one of $x_i$ $\rho$-average splits $V$. If $\delta_0 \leq \delta/(2m (2 + \log(1/\rho)))$, then with probability at least $1- \delta$, {\sc select} returns a point $x_i$ that $(\rho/8)$-average splits $V$, finishing in less than $\lceil \log(1/\rho) \rceil + 1$ rounds and sampling $O\left( \left( \frac{1}{\epsilon \rho^2} + \frac{\log(1/\rho)}{\Phi(V)^2} \right) \log \frac{1}{\delta_0} \right)$ hypotheses in total.
\end{theorem} 

\begin{figure*}[ttt!]
\begin{minipage}[t]{3.12in}
\begin{algorithm}[H]
 \caption{{\textsc{dbal}}}
 \label{algorithm: average splitting index algorithm}
\begin{algorithmic}
   \STATE {\bfseries Input:} Hypothesis class $\HH$, prior distribution $\pi$
   \STATE Initialize $V = \HH$
   \WHILE{$\frac{1}{n} \psi(E) \geq \frac{3\epsilon}{4}$ for $E \sim( \pi|_V )^{2 \times n}$}
   		\STATE Draw $m$ data points $\mathbf{x} = (x_1, \ldots, x_m)$
		\STATE Query point $x_i = \textsc{select}(V, \mathbf{x})$ and set $V$ to be consistent with the result
	\ENDWHILE
	\RETURN Current version space $V$ in the form of the queried points $(x_1, h^*(x_1)), \ldots, (x_K, h^*(x_K))$
\end{algorithmic}
\end{algorithm}
\end{minipage}
\hfill
\begin{minipage}[t]{3.48in}
\begin{algorithm}[H]
 \caption{{\textsc{select}}}
 \label{algorithm: query selection procedure}
\begin{algorithmic}
   \STATE {\bfseries Input:} Version space $V$, prior $\pi$, data $\mathbf{x}= (x_1, \ldots, x_m)$ \\
   \STATE Set $\widehat{\rho}_1 = 1/2$
   \FOR{$t = 1, 2, \ldots$}
   		\STATE Draw $E' \sim (\pi|_V)^{2 \times m_t}$ and compute $\widehat{\Phi}_t = \frac{1}{m_t} \psi(E')$
   		\STATE Draw $E \sim (\pi|_V)^{2 \times n_t}$
   		\STATE If $\exists \, x_i$ s.t. $\frac{1}{n_t}\max \left\{ \psi(E_{x_i}^+), \psi(E_{x_i}^-)\right\} \leq (1- \widehat{\rho}_t) \widehat{\Phi}_t$, then \textbf{halt} and \textbf{return} $x_i$
		\STATE Otherwise, let $\widehat{\rho}_{t+1} = \widehat{\rho}_t/2$
	\ENDFOR
\end{algorithmic}
\end{algorithm}
\end{minipage}
\end{figure*}

\paragraph{Remark 1:} It is possible to modify {\sc select} to find a point $x_i$ that $(c \rho)$-average splits $V$ for any constant $c < 1$ while only having to draw $O(1)$ more hypotheses in total. First note that by halving $\widehat{\rho}_t$ at each step, we immediately give up a factor of two in our approximation. This can be made smaller by taking narrower steps. Additionally, with a constant factor increase in $m_t$ and $n_t$, the approximation ratios in Lemma~\ref{lemma: selection procedure lemma} can be set to any constant.

\paragraph{Remark 2:} At first glance, it appears that {\sc select} requires us to know $\rho$ in order to calculate $\delta_0$. However, a crude lower bound on $\rho$ suffices. Such a bound can always be found in terms of $\epsilon$. This is because any version space is $(\epsilon/2, \epsilon, \epsilon/2)$-splittable~\cite[Lemma~1]{dasgupta-splitting-index}. By Lemma~\ref{lemma: splitting index implies average splitting index lemma}, so long as $\tau$ is less than $\epsilon/4$, we can substitute $\frac{\epsilon}{8\lceil\log(2/\epsilon)\rceil}$ for $\rho$ in when we compute $\delta_0$.


\begin{proof}[Proof of Theorem~\ref{theorem: selection procedure theorem}]
Let $T := \lceil \log(1/\rho) \rceil + 1$. By Lemma~\ref{lemma: selection procedure lemma}, we know that for rounds $t=1,\ldots, T$, we don't return any point which does worse than $\widehat{\rho}_t/2$-average splits $V$ with probability $1-\delta/2$. Moreover, in the $T$-th round, it will be the case that $\rho/4 \leq  \widehat{\rho}_T \leq \rho/2$, and therefore, with probability $1-\delta/2$, we will select a point which does no worse than $\widehat{\rho}_T/2$-average split $V$, which in turn does no worse than $\rho/8$-average split $V$.

Note that we draw $m_t + n_t$ hypotheses at each round. By Lemma~\ref{lemma: selection procedure lemma}, for each round $\widehat{\Phi}_t \geq 3\Phi(V)/4 \geq 3\epsilon/4$. Thus
\begin{align*}
\# \text{ of hypotheses drawn } \ = \ \sum_{t=1}^T \left(\frac{48}{\widehat{\rho}_t^2 \epsilon} + \frac{32}{\widehat{\rho}_t^2 \widehat{\Phi}_t} + \frac{40}{\widehat{\Phi}_t^2} \right) \log\frac{4}{\delta_0}  \ \leq \ \sum_{t=1}^T \left(\frac{96}{ \epsilon \widehat{\rho}_t^2} + \frac{72}{\Phi(V)^2}\right) \log  \frac{4}{\delta_0}
\end{align*}
Given $\widehat{\rho}_t = 1/2^t$ and $T \leq 2 + \log 1/\rho$, we have
\[ \sum_{t=1}^T \frac{1}{\widehat{\rho}_t^2} \ = \  \sum_{t=1}^T 2^{2t} \ \leq \ \left(\sum_{t=1}^T 2^{t} \right)^2 \ \leq \ \left(2^{2 + \log 1/\rho} \right)^2 \ = \ \frac{16}{\rho^2}. \]
Plugging in  $\delta_0 \leq \frac{\delta}{2m (2 + \log(1/\rho))}$, we recover the theorem statement.
\end{proof}

\subsection{Active Learning Strategy}
\label{section: active learning strategy}

Using the \textsc{select} procedure as a subroutine, Algorithm~\ref{algorithm: average splitting index algorithm}, henceforth DBAL for Diameter-based Active Learning, is our active learning strategy. Given a hypothesis class with average splitting index $(\rho, \epsilon/2, \tau)$, DBAL queries data points provided by \textsc{select} until it is confident $\Phi(V) < \epsilon$. 

Denote by $V_t$ the version space in the $t$-th round of DBAL. The following lemma, which is proven in the appendix, demonstrates that the halting condition (that is, $\psi(E) < 3\epsilon n/4$, where $E$ consists of $n$ pairs sampled from $(\pi|_V)^2$) guarantees that with high probability DBAL stops when $\Phi(V_t)$ is small.

\begin{restatable}{lemma}{CorrectTerminationLemma}
\label{lemma: correct termination lemma}
The following holds for DBAL:
\begin{itemize}
	\item[(a)] Suppose that for all $t = 1, 2, \ldots, K$ that $\Phi(V_t) > \epsilon$. Then the probability that the termination
		condition is ever true for any of those rounds is bounded above by $K \exp	\left( - \frac{\epsilon n }{32} \right)$.
	\item[(b)] Suppose that for some $t = 1, 2, \ldots, K$ that $\Phi(V_t) \leq \epsilon/2$. Then the probability that the termination
		condition is not true in that round is bounded above by $K \exp \left( - \frac{\epsilon n }{48} \right)$.
\end{itemize}
\end{restatable}

Given the guarantees on the \textsc{select} procedure in Theorem~\ref{theorem: selection procedure theorem} and on the termination condition provided by Lemma~\ref{lemma: correct termination lemma}, we get the following theorem.

\begin{theorem}
\label{theorem: average splittability implies good algorithm theorem}
Suppose that $\HH$ has average splitting index $(\rho, \epsilon/2, \tau)$. Then DBAL returns a version space $V$ satisfying $\Phi(V) \leq \epsilon$ with probability at least $1 - \delta$ while using the following resources:
\begin{itemize}
	\item[(a)] $K \leq \frac{8}{\rho} \left( \log \frac{2}{\epsilon} + 2 \log \frac{1}{\pi(h^*)} \right)$ rounds, with one label per round,
	\item[(b)] $m \leq \frac{1}{\tau} \log  \frac{2K}{\delta} $ unlabeled data points sampled per round, and
	\item[(c)] $n \leq O \! \left( \left( \frac{1}{\epsilon \rho^2} + \frac{\log(1/\rho)}{\epsilon^2} \right) \left( \log \frac{mK}{\delta} + \log \log \frac{1}{\epsilon} \right)\right)$ hypotheses sampled per round.
\end{itemize}
\end{theorem}
\begin{proof}
From definition of the average splitting index, if we draw $m = \frac{1}{\tau} \log  \frac{2K}{\delta}$ unlabeled points per round, then with probability $1-\delta/2$, each of the first $K$ rounds will have at least one data point that $\rho$-average splits the current version space. In each such round, if the version space has average diameter at least $\epsilon/2$, then with probability $1-\delta/4$ \textsc{select} will return a data point that $\rho/8$-average splits the current version space while sampling no more than $n = O\left( \left(\frac{1}{\epsilon \rho^2} + \frac{1}{\epsilon^2} \log \frac{1}{\rho} \right) \log \frac{mK \log\frac{1}{\epsilon}}{\delta} \right)$ hypotheses per round by Theorem~\ref{theorem: selection procedure theorem}.

By Lemma~\ref{lemma: correct termination lemma}, if the termination check uses $n'= O\left( \frac{1}{\epsilon}\log \frac{1}{\delta} \right)$ hypotheses per round, then with probability $1-\delta/4$ in the first $K$ rounds the termination condition will never be true when the current version space has average diameter greater than $\epsilon$ and will certainly be true if the current version space has diameter less than $\epsilon/2$.

Thus it suffices to bound the number of rounds in which we can $\rho/8$-average split the version space before encountering a version space with $\epsilon/2$. 

Since the version space is always consistent with the true hypothesis $h^*$, we will always have $\pi(V_t) \geq \pi(h^*)$. After $K = \frac{8}{\rho} \left( \log \frac{2}{\epsilon} + 2 \log \frac{1}{\pi(h^*)} \right)$ rounds of $\rho/8$-average splitting, we have
\begin{align*}
\pi(h^*)^2 \Phi(V_K) \ \leq \ \pi(V_K)^2\Phi(V_K)  \ \leq \ \left(1- \frac{\rho}{8} \right)^K \pi(V_0)^2 \Phi(V_0) \ \leq \ \frac{\pi(h^*)^2\epsilon}{2}
\end{align*}
Where we have used the fact that $\pi(V)^2 \Phi(V) \leq 1$ for any set $V \subset \HH$. Thus in the first $K$ rounds, we must terminate with a version space with average diameter less than $\epsilon$.
\end{proof}

\section{Proof of Lemma \ref{lemma: splitting index implies average splitting index lemma}}
\label{section: Proof}
In this section, we give the proof of the following relationship between the original splitting index and our average splitting index.
\SplittingImpliesAverageSplittingLemma*
The first step in proving Lemma~\ref{lemma: splitting index implies average splitting index lemma} is to relate the splitting index to our estimator $\psi(\cdot)$. Intuitively, splittability says that for any set of large edges there are many data points which remove a significant fraction of them. One may suspect this should imply that if a set of edges is large on average, then there should be many data points which remove a significant fraction of their weight. The following lemma confirms this suspicion.

\begin{lemma}
\label{lemma: splitting index implies finite average splitting index lemma}
Suppose that $V \subset \HH$ has splitting index $(\rho, \epsilon, \tau)$, and say $E = (\{h_1, h_1' \}, \ldots, \{ h_n, h_n' \})$ is a sequence of hypothesis pairs from $V$ satisfying $\frac{1}{n} \psi(E) > 2 \epsilon$. Then if $x \sim \D$, we have with probability at least $\tau$, 
\[ \max \left\{ \psi(E_x^+), \psi(E_x^-) \right\} \leq \left(1 - \frac{\rho}{4 \lceil \log(1/\epsilon) \rceil}\right)\psi(E) .\]
\end{lemma}
\begin{proof}
Consider partitioning $E$ as
\begin{align*}
E_0 &= \{ \{ h, h' \} \in E \, : \, d(h,h') < \epsilon \} \text{ and } \\
E_k &= \{ \{ h, h' \} \in E \, : \, d(h,h') \in [2^{k-1} \epsilon, 2^k \epsilon)
\end{align*}
for $k=1,\ldots, K$ with $K =  \lceil \log \frac{1}{\epsilon} \rceil$. Then $E_0, \ldots, E_K$ are all disjoint and their union is $E$. Define $E_{1:K} = \cup_{k=1}^K E_k$. 

We first claim that $\psi(E_{1:K}) > \psi(E_0)$. This follows from the observation that because $\psi(E) \geq 2 n  \epsilon$ and each edge in $E_0$ has length less than $\epsilon$, we must have
\[ \psi(E_{1:K}) \ = \ \psi(E) - \psi(E_0) \ > \ 2 n \epsilon - n \epsilon \ > \ \psi(E_0). \]
Next, observe that because each edge $\{ h, h' \} \in E_k$ with $k \geq 1$ satisfies $d(h,h') \in [2^{k-1}\epsilon, 2^k \epsilon)$, we have
\begin{align*}
\psi(E_{1:K})  \ = \ \sum_{k=1}^K \sum_{\{ h , h' \} \in E_k}  d(h,h') \ \leq \ \sum_{k=1}^K  2^{k} \epsilon |E_k|.
\end{align*}
Since there are only $K$ summands on the right, at least one of these must be larger than $\psi(E_{1:K})/K$. Let $k$ denote that index and let $x$ be a point which $\rho$-splits $E_k$. Then we have
\begin{align*}
\psi((E_{1:K})^+_x) \ &\leq \ \psi(E_{1:K}) - \psi(E_k \setminus (E_k)_x^+) \\ 
&\leq \ \psi(E_{1:K}) - \rho 2^{k-1} \epsilon |E_k| \\
&\leq \ \left( 1 - \frac{\rho}{2K} \right) \psi(E_{1:K}). 
\end{align*}
Since $\psi(E_{1:K}) \geq \psi(E_0)$, we have
\begin{align*}
\psi(E^+_x) \ \leq \ \psi(E_0) + \left( 1 - \frac{\rho}{2K} \right) \psi(E_{1:K}) \ \leq \ \left( 1 - \frac{\rho}{4K} \right) \psi(E).
\end{align*}
Symmetric arguments show the same holds for $E^-_{x}$. 

Finally, by the definition of splitting, the probability of drawing a point $x$ which $\rho$-splits $E_k$ is at least $\tau$, giving us the lemma.
\end{proof}

With Lemma~\ref{lemma: splitting index implies finite average splitting index lemma} in hand, we are now ready to prove Lemma~\ref{lemma: splitting index implies average splitting index lemma}.
\begin{figure}[b]
\centering
\includegraphics[scale=.75]{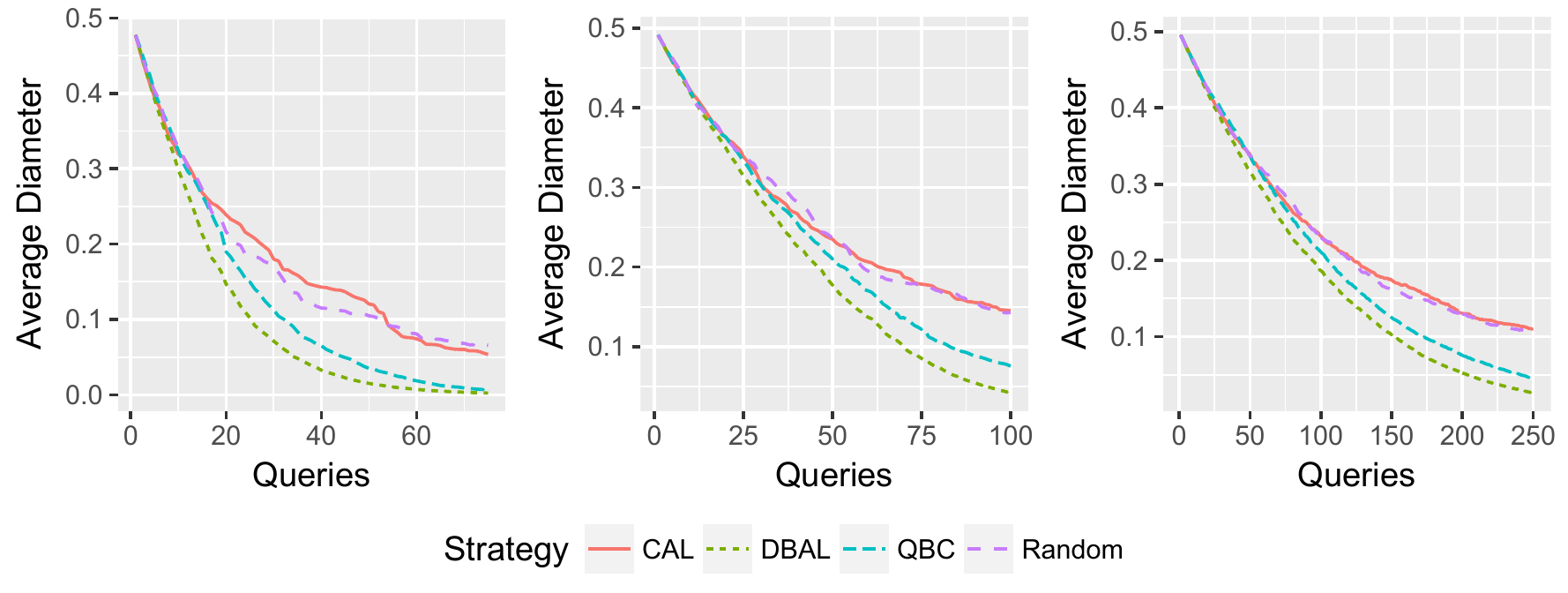}
\caption{Simulation results on homogeneous linear separators. \textit{Left}: $d = 10$. \textit{Middle}: $d = 25$.   \textit{Right}: $d = 50$. \label{figure: homogeneous linear separators figure}}
\end{figure}
\begin{proof}[Proof of Lemma~\ref{lemma: splitting index implies average splitting index lemma}]
Let $V \subset \HH$ such that $\Phi(V) > 2\epsilon$. Suppose that we draw $n$ edges $E$ i.i.d. from $\pi|_V$ and draw a data point $x \sim \D$. Then Hoeffding's inequality \cite{hoeffding-inequality}, combined with Lemma~\ref{lemma: expectation lemma}, tells us that there exist sequences $\epsilon_n, \delta_n \searrow 0$ such that with probability at least $1-3\delta_n$, the following hold simultaneously:
\begin{itemize}
	\item $\Phi(V) - \epsilon_n \ \leq \ \frac{1}{n}\psi(E) \ \leq \ \Phi(V) + \epsilon_n$,
	\item $\frac{1}{n} \psi(E_x^+) \ \geq \ \frac{\pi(V_x^+)^2}{\pi(V)^2}\Phi(V_x^+) - \epsilon_n$, and
	\item $\frac{1}{n} \psi(E_x^-) \ \geq \ \frac{\pi(V_x^-)^2}{\pi(V)^2}\Phi(V_x^-) - \epsilon_n$.
\end{itemize}
For $\epsilon_n$ small enough, we have that $\Phi(V) - \epsilon_n > 2\epsilon$. Combining the above with Lemma~\ref{lemma: splitting index implies finite average splitting index lemma}, we have with probability at least $\tau - 3\delta_n$,
\begin{align*}
\max \left\{ \frac{\pi(V_x^+)^2}{\pi(V)^2}\Phi(V_x^+), \frac{\pi(V_x^-)^2}{\pi(V)^2}\Phi(V_x^-)  \right\} - \epsilon_n \ &\leq \ \frac{1}{n} \max \{ \psi(E_x^+), \psi(E_x^-)  \} \\
&\leq \ \left(1 - \frac{\rho}{4 \lceil \log(1/\epsilon) \rceil}\right) \frac{\psi(E)}{n} \\
&\leq \ \left(1 - \frac{\rho}{4 \lceil \log(1/\epsilon) \rceil}\right)(\Phi(V) + \epsilon_n).
\end{align*}
By taking $n \rightarrow \infty$, we have $\epsilon_n, \delta_n \searrow 0$, giving us the lemma.
\end{proof}
\section{Simulations} 
\label{section: Simulations}
We compared DBAL against the baseline passive learner as well as two other generic active learning strategies: CAL and QBC. CAL proceeds by randomly sampling a data point and querying it if its label cannot be inferred from previously queried data points. QBC uses a prior distribution $\pi$ and maintains a version space $V$. Given a randomly sampled data point $x$, QBC samples two hypotheses $h, h' \sim \pi|_V$ and queries $x$ if $h(x) \neq h'(x)$.

We tested on two hypothesis classes: homogeneous, or through-the-origin, linear separators and $k$-sparse monotone disjunctions. In each of our simulations, we drew our target $h^*$ from the prior distribution. After each query, we estimated the average diameter of the version space. We repeated each simulation several times and plotted the average performance of each algorithm.

\paragraph{Homogeneous linear separators} The class of $d$-dimensional homogeneous linear separators can be identified with elements of the $d$-dimensional unit sphere. That is, a hypothesis $h \in \mathcal{S}^{d-1}$ acts on a data point $x \in \R^d$ via the sign of their inner product:
\[ h(x) \ := \ \text{sign}(\langle h, x \rangle). \]
In our simulations, both the prior distribution and the data distribution are uniform over the unit sphere. Although there is no known method to exactly sample uniformly from the version space, Gilad-Bachrach et al.~\cite{gilad-bachrach-nt-QBC-real} demonstrated that using samples generated by the hit-and-run Markov chain works well in practice. We adopted this approach for our sampling tasks.

Figure~\ref{figure: homogeneous linear separators figure} shows the results of our simulations on homogeneous linear separators. 

\begin{figure}
\centering
\includegraphics[scale=.80]{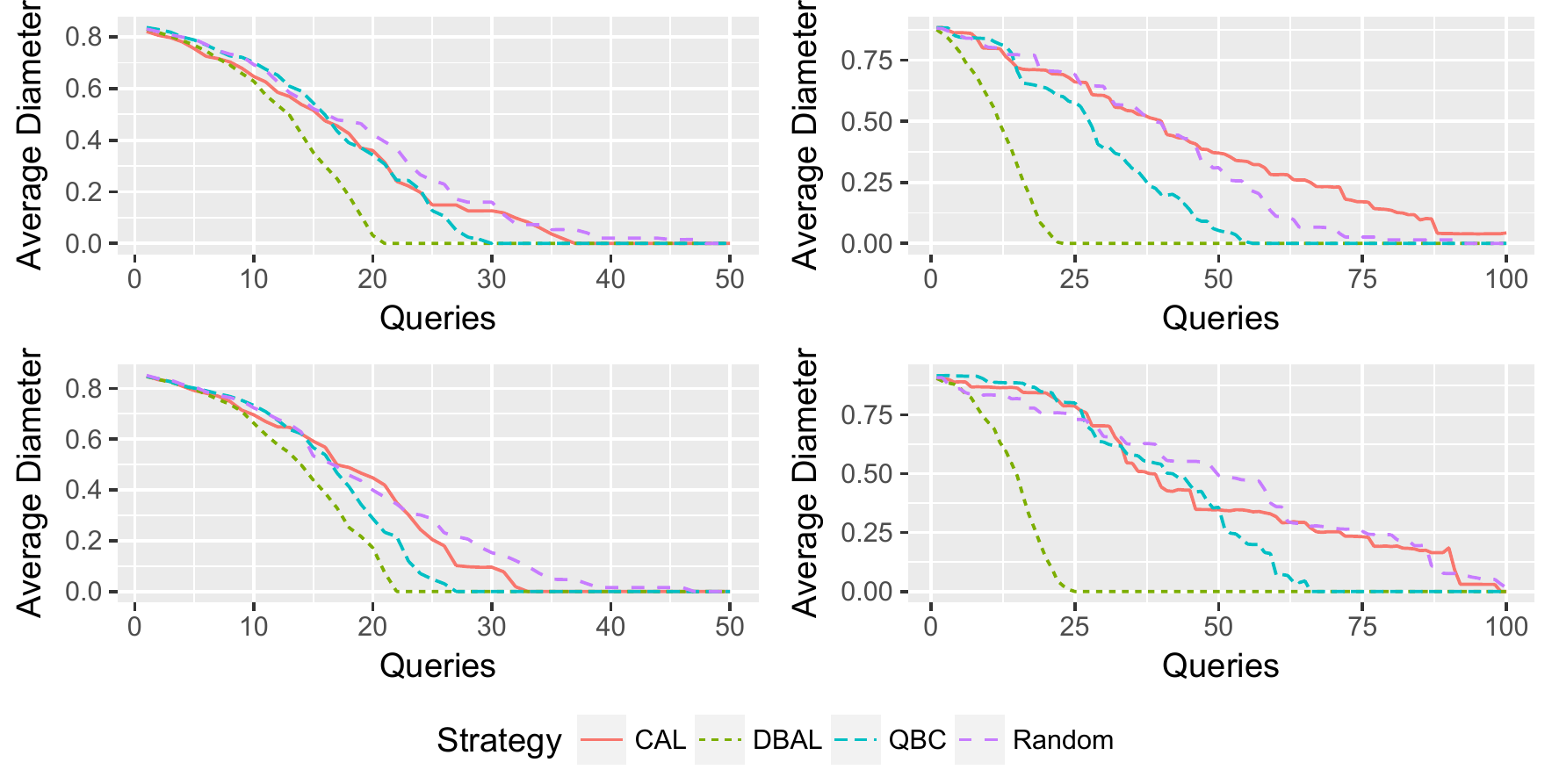}
\caption{Simulation results on $k$-sparse monotone disjunctions. In all cases $k = 4$. \textit{Top left}: $d = 75$, $p=0.25$.  \textit{Top right}: $d = 75$, $p=0.5$. \textit{Bottom left}: $d = 100$, $p=0.25$.  \textit{Bottom right}: $d = 100$, $p=0.5$. \label{figure: k sparse monotone disjunctions figure}}
\end{figure}

\paragraph{Sparse monotone disjunctions} A $k$-sparse monotone disjunction is a disjunction of $k$ positive literals. Given a Boolean vector $x \in \{0, 1\}^n$, a monotone disjunction $h$ classifies $x$ as positive if and only if $x_i = 1$ for some positive literal $i$ in $h$.

In our simulations, each data point is a vector whose coordinates are i.i.d. Bernoulli random variables with parameter $p$. The prior distribution is uniform over all $k$-sparse monotone disjunctions. When $k$ is constant, it is possible to sample from the prior restricted to the version space in expected polynomial time using rejection sampling.

The results of our simulations on $k$-sparse monotone disjunctions are in Figure~\ref{figure: k sparse monotone disjunctions figure}. 

\section*{Acknowledgments} 

The authors are grateful to the NSF for support under grants IIS-1162581 and DGE-1144086. Part of this work was done at the Simons Institute for Theoretical Computer Science, Berkeley, as part of a program on the foundations of machine learning. CT additionally thanks Daniel Hsu and Stefanos Poulis for helpful discussions.

\bibliography{references}
\bibliographystyle{plain}

\clearpage

\section*{Appendix: Proof Details}

\subsection*{Remark from Section~\ref{section: Preliminaries}}

In Section~\ref{section: Preliminaries}, the remark after the definition of average splitting stated that there exist hypothesis classes $V$ for which there are many points which $1/4$-split $E$ for any $E \subset {V \choose 2}$ but for which any $x \in \X$ satisfies 
\[ \max \{ \Phi(V_x^+), \Phi(V_x^-) \} \approx \Phi(V).\]
Here we formally prove this statement.

Consider the hypothesis class of homogeneous linear separators and let $V = \{e_1, \ldots, e_n\} \subset \HH$ where $e_k$ is the $k$-th unit coordinate vector. Let the data distribution be uniform over the $n$-sphere and the prior distribution $\pi$ be uniform over $V$. As a subset of the homogeneous linear separators, $V$ has splitting index $(1/4, \epsilon, \Theta(\epsilon))$ \cite[Theorem~10]{dasgupta-splitting-index}.

On the other hand, for any $i \neq j$, $d(h_i, h_j) = 1/2$. This implies that 
\[ \Phi(V) \ = \ \pr(h \neq h') \E_{h,h'}[d(h,h') \, | \, h\neq h'] \ = \ \frac{n-1}{2n}. \]
Moreover, any query $x \in \X$ eliminates at most half the hypotheses in $V$ in the worst case. Therefore, for all $x \in \X$,
$$\max \{ \Phi(V_x^+), \Phi(V_x^-) \} \ \geq \ \frac{(n/2 - 1)}{2 (n/2)} 
\ = \ \left(\frac{n-2}{n-1}\right) \Phi(V).$$

\subsection*{Proofs of Lemma~\ref{lemma: selection procedure lemma} and Lemma \ref{lemma: correct termination lemma}}

The proofs in this section rely crucially on two concentration inequalities. The first is due to Hoeffding~\cite{hoeffding-inequality}.

\begin{lemma}[Hoeffding 1963]
\label{lemma: Hoeffding bound lemma}
Let $X_1, \ldots, X_n$ be i.i.d. random variables taking values in $[0,1]$ and let $X = \sum X_i$ and $\mu = \E[X]$. Then for $t > 0$,
\[ \pr(X - \mu \geq t) \leq \exp \left( -\frac{2t^2}{n} \right) \]
\end{lemma}
Our other tool will be the following multiplicative Chernoff-Hoeffding bound due to Angluin and Valiant \cite{angluin-v-multiplicative-chernoff}.
\begin{lemma}[Angluin and Valiant 1977]
\label{lemma: multiplicative Chernoff bound lemma}
Let $X_1, \ldots, X_n$ be i.i.d. random variables taking values in $[0,1]$ and let $X = \sum X_i$ and $\mu = \E[X]$. Then for $0 < \beta < 1$,
\begin{itemize}
	\item[(i)] $\pr(X \leq (1 - \beta)\mu) \leq \exp\left(	- \frac{\beta^2 \mu}{2} \right)$ and
	\item[(ii)] $\pr(X \geq (1 + \beta)\mu) \leq \exp\left(	- \frac{\beta^2 \mu}{3} \right)$.
\end{itemize}
\end{lemma}
We now turn to the proof of Lemma~\ref{lemma: selection procedure lemma}.
\SelectionProcedureLemma*
\begin{proof}
In round $t$, let $\widehat{\rho} := \widehat{\rho}_t$, $\widehat{\Phi} := \widehat{\Phi}_t$, $m := m_t$, and $n := n_t$. 

\medskip

For (a), recall $\widehat{\Phi} = \frac{1}{m} \psi(E')$ for $E' \sim (\pi|_V)^{2 \times m}$. By Lemma~\ref{lemma: multiplicative Chernoff bound lemma}, we have for $\beta_0 >0$ 
\[ \pr \left((1-\beta_0)\Phi(V) \leq \widehat{\Phi} \leq (1+\beta_0)\Phi(V) \right) \geq 1- 2\exp\left( - \frac{m \beta_0^2 \epsilon}{3} \right). \]
Taking $m \geq \frac{3}{\beta_0^2 \epsilon} \log\left( \frac{4}{\delta_0} \right)$, we have the above probability is at least $1 - \delta_0/2$. Let us condition on this event occurring.
\medskip

To see (b), say w.l.o.g. $\left( \frac{\pi(V_x^+)}{\pi(V)} \right)^2 \Phi(V_x^+) = (1-\rho)\Phi(V)$. Then, we have
\begin{align*}
\pr\left(\frac{1}{n}\psi(E_x^+) \leq (1-\widehat{\rho}) \widehat{\Phi} \right) &\leq \pr\left(\frac{1}{n}\psi(E_x^+) \leq (1-\widehat{\rho}) (1+ \beta_0) \Phi(V) \right).
\end{align*}
Taking $\beta$ such that $(1-\beta)(1-\rho) = (1-\widehat{\rho})(1+\beta_0)$, we have by Lemma~\ref{lemma: multiplicative Chernoff bound lemma}~(i),
\begin{align*}
\pr\left(\frac{1}{n}\psi(E_x^+) \leq (1-\widehat{\rho}) \widehat{\Phi} \right) &\leq \pr\left(\frac{1}{n}\psi(E_x^+) \leq (1-\beta) (1-\rho) \Phi(V) \right) \\
&\leq \exp\left( - \frac{n \beta^2(1-\rho) \Phi(V)}{2} \right)\\
&\leq \exp  \left( - \frac{n (1-\rho) \widehat{\Phi}}{2(1+\beta_0)} \cdot \left[ 1 - \frac{(1 - \widehat{\rho})(1+ \beta_0)}{1-\rho} \right]^2  \right) \\
&\leq \exp  \left( - \frac{n (1-\widehat{\rho}/2) \widehat{\Phi}}{2(1+\beta_0)} \cdot \left[ 1 - \frac{(1 - \widehat{\rho})(1+ \beta_0)}{1-\widehat{\rho}/2} \right]^2  \right).
\end{align*}
Taking $\beta_0 \leq \widehat{\rho}/4$, the above is less than $\exp  \left( - \frac{n \widehat{\Phi} \widehat{\rho}^2}{32}  \right).$ With $n$ as in the lemma statement and combined with our results on the concentration of $\widehat{\Phi}$, we have that with probability $1-\delta_0$ \[ \frac{1}{n}\max \left\{ \psi(E_x^+), \psi(E_x^-)\right\} > (1-\widehat{\rho}) \widehat{\Phi}. \]

To see (c),  suppose now that w.l.o.g. $\left( \frac{\pi(V_x^-)}{\pi(V)} \right)^2 \Phi(V_x^-)  \leq  \left( \frac{\pi(V_x^+)}{\pi(V)} \right)^2 \Phi(V_x^+) = (1-\rho)\Phi(V)$. We need to consider two cases.

\paragraph{Case 1: $\rho \leq 1/2$.} 
Taking $\beta$ such that $(1+\beta)(1-\rho) = (1-\widehat{\rho})(1 - \beta_0)$, we have by Lemma~\ref{lemma: multiplicative Chernoff bound lemma}~(ii),
\begin{align*}
\pr\left(\frac{1}{n}\psi(E_x^+) > (1-\widehat{\rho}) \widehat{\Phi} \right) &\leq \pr\left(\frac{1}{n}\psi(E_x^+) > (1-\widehat{\rho}) (1-\beta_0) \Phi(V) \right) \\
&= \pr\left(\frac{1}{n}\psi(E_x^+) > (1+\beta) (1-\rho) \Phi(V) \right) \\
&\leq \exp \left( - \frac{n \beta^2 (1- \rho)\Phi(V)}{3} \right) \\
&\leq \exp \left( - \frac{n (1- \rho)\widehat{\Phi}}{3(1+\beta_0)} \cdot \left[\frac{(1-\widehat{\rho})(1-\beta_0)}{1-\rho} -1\right]^2 \right) \\
&\leq \exp \left( - \frac{n \widehat{\Phi}}{6(1+\beta_0)} \cdot \left[\frac{(1-\widehat{\rho})(1-\beta_0)}{1- 2 \widehat{\rho}} -1 \right]^2 \right).
\end{align*}

Taking $\beta_0 \leq \widehat{\rho}/4$, the above is less than $\exp  \left( - \frac{n \widehat{\Phi} \widehat{\rho}^2}{12}  \right)$. Because $\left( \frac{\pi(V_x^-)}{\pi(V)} \right)^2 \Phi(V_x^-)  \leq  \left( \frac{\pi(V_x^+)}{\pi(V)} \right)^2 \Phi(V_x^+)$, we also can say 
\[ \pr\left(\frac{1}{n}\psi(E_x^-) > (1-\widehat{\rho}) \widehat{\Phi} \right) \leq \exp  \left( - \frac{n \widehat{\Phi} \widehat{\rho}^2}{12}  \right). \]

\paragraph{Case 2:  $\rho > 1/2$.} Taking $\beta_0 \leq 1/16$, we have 
\begin{align*}
\pr\left(\frac{1}{n} \psi(E_x^+) > (1- \widehat{\rho})\widehat{\Phi} \right) 
&\leq \pr\left(\frac{1}{n} \psi(E_x^+) > (1- \widehat{\rho})(1-\beta_0)\Phi(V) \right) \\
&= \pr\left(\frac{1}{n} \psi(E_x^+) > (1-\rho)\Phi(V) +  ((1- \widehat{\rho})(1-\beta_0) - (1-\rho)) \Phi(V) \right) \\
&\leq \pr\left(\frac{1}{n} \psi(E_x^+) > (1-\rho)\Phi(V) +  (\rho - \widehat{\rho} - \beta_0 ) \Phi(V) \right) \\
&\leq \pr\left(\frac{1}{n} \psi(E_x^+) > (1-\rho)\Phi(V) +  \left(\frac{\rho}{2} - \beta_0 \right) \Phi(V) \right) \\
&\leq \pr\left(\frac{1}{n} \psi(E_x^+) > (1-\rho)\Phi(V) +  \left( \frac{1}{4} - \beta_0 \right) \Phi(V) \right) \\
&\leq \pr\left(\frac{1}{n} \psi(E_x^+) > (1-\rho)\Phi(V) +  \frac{\frac{1}{4} - \beta_0}{1 + \beta_0}  \widehat{\Phi}\right) \\
&\leq \pr\left(\frac{1}{n} \psi(E_x^+) > (1-\rho)\Phi(V) +  \frac{3}{17}  \widehat{\Phi}\right) 
\end{align*}
By Lemma~\ref{lemma: Hoeffding bound lemma}, the above is less than $\exp\left(- \frac{n \widehat{\Phi}^2}{40}  \right)$. Because $\left( \frac{\pi(V_x^-)}{\pi(V)} \right)^2 \Phi(V_x^-)  \leq  \left( \frac{\pi(V_x^+)}{\pi(V)} \right)^2 \Phi(V_x^+)$, we also can say 
\[ \pr\left(\frac{1}{n}\psi(E_x^-) > (1-\widehat{\rho}) \widehat{\Phi} \right) \leq \exp\left(- \frac{n \widehat{\Phi}^2}{40}  \right). \]

Regardless of which case we are in, we have for $n$ as in the lemma statement, with probability $1- \delta_0$, \[ \frac{1}{n}\max \left\{ \psi(E_x^+), \psi(E_x^-)\right\} \leq (1- \widehat{\rho}) \widehat{\Phi}. \qedhere \]
\end{proof}

We next provide the proof of Lemma~\ref{lemma: correct termination lemma}.
\CorrectTerminationLemma*
\begin{proof}
Recall that the termination condition from DBAL is $\frac{1}{n} \psi(E) < \frac{3\epsilon}{4}$ for $E \sim (\pi|_V)^{2\times n}$.

Part (a) follows from plugging in $\beta = \frac{1}{4}$ into Lemma~\ref{lemma: multiplicative Chernoff bound lemma}~(i) and taking a union bound over rounds $1, \ldots, K$.

Similarly, part (b) follows from plugging in $\beta = \frac{1}{4}$ into Lemma~\ref{lemma: multiplicative Chernoff bound lemma}~(ii) and taking a union bound over rounds $1, \ldots, K$.
\end{proof}

\end{document}